\documentclass{article}

     \usepackage[preprint]{neurips_2024}

\usepackage[utf8]{inputenc} 
\usepackage[T1]{fontenc}    
\usepackage{hyperref}       
\usepackage{url}            
\usepackage{booktabs}       
\usepackage{amsfonts}       
\usepackage{nicefrac}       
\usepackage{microtype}      
\usepackage{xcolor}         
\usepackage[utf8]{inputenc} 
\usepackage[T1]{fontenc}    
\usepackage{hyperref}       
\usepackage{url}            
\usepackage{booktabs}       
\usepackage{amsfonts}       
\usepackage{nicefrac}       
\usepackage{microtype}      
\usepackage{xcolor}         

\usepackage{mathtools}
\usepackage{amssymb}
\usepackage{amsthm}

\theoremstyle{plain}
\newtheorem{theorem}{Theorem}[section]

\newtheorem{lemma}[theorem]{Lemma}

\theoremstyle{definition}
\newtheorem{definition}[theorem]{Definition}

\theoremstyle{remark}

\newcommand{\inner}[1]{\langle #1\rangle}

\title{Tighter Risk Bounds for Mixtures of Experts}

\author{%
  Wissam Akretche\\
  Département d’informatique et de génie logiciel\\
  Université Laval\\
  Québec, G1V 0A6, Canada \\
  \texttt{wissam.akretche.1@ulaval.ca} \\
  \And
  Frédéric LeBlanc \\
  Institut intelligence et données\\
  Université Laval \\
  Québec, G1V 0A6, Canada \\
  \texttt{frederic.leblanc@iid.ulaval.ca} \\
  \AND
  Mario Marchand\\
  Département d’informatique et de génie logiciel\\
  Université Laval\\
  Québec, G1V 0A6, Canada \\
  \texttt{mario.marchand@ift.ulaval.ca} \\
}

\usepackage{def}
\newcommand{\mVert}{\mkern2mu\Vert\mkern2mu}

\DeclareMathOperator*{\ev}{\mathbb{E}}
\DeclareMathOperator*{\prob}{\mathbb{P}}
\DeclareMathOperator{\ReLU}{ReLU}
\newcommand{\mvert}{\mkern2mu\vert\mkern2mu}

\begin{document}

\maketitle

\begin{abstract}
  In this work, we provide upper bounds on the risk of mixtures of experts by imposing local differential privacy (LDP) on their gating mechanism. These theoretical guarantees are  tailored to mixtures of experts that utilize the one-out-of-$n$ gating mechanism, as opposed to the conventional $n$-out-of-$n$ mechanism. The bounds exhibit logarithmic dependence on the number of experts, and encapsulate the dependence on the gating mechanism in the LDP parameter, making them significantly tighter than existing bounds, under reasonable conditions. Experimental results support our theory, demonstrating that our approach enhances the generalization ability of mixtures of experts and validating the feasibility of imposing LDP on the gating mechanism.
\end{abstract}

\section{Introduction}
Mixtures of experts, initially introduced by \cite{adaptive_mixture_of_local_experts}, have found widespread use in modeling sequential data, including applications in classification, regression, pattern recognition and feature selection tasks (\cite{feature_selection_van} and \cite{feature_selection_khalili}).  One of the fundamental motivations behind mixtures of experts is their ability to break down complex problems into more manageable sub-problems, potentially simplifying the overall task. The structure of these models is well suited to capturing unobservable heterogeneity in the data generation process, dealing with this problem by splitting the data into homogeneous subsets (with the gating network) and associating each subset with an expert. 
This intuitive architecture has led to significant interest in mixture of experts models, resulting in a wealth of research (\cite{twenty_years_moe}), ranging from simple mixtures of experts ( \cite{adaptive_mixture_of_local_experts, moe_hierarchical}) to sparsely gated models (\cite{sparsly_gated_moe}). Moreover, this architecture has inspired the development of various other models, such as switch transformers (\cite{switch_transformers}).
However, despite the considerable attention mixtures of experts have received, advancements in their theoretical analysis have been relatively limited. \cite{risk_bound_moe_azran_ron} proved data-dependent risk bounds for mixtures of experts (with the $n$-out-of-$n$ gating mechanism) using Rademacher complexity, but they exhibit a dependence on the complexity of the class of gating networks and the sum of the complexities of the expert classes, which reflects the complex structure of mixtures of experts but unfortunately leads to potentially large bounds. We are not aware of other work proving generalization bounds specifically tailored to mixtures of experts.

To make theoretical progress, we utilize a well-known privacy-preserving technique called Local Differential Privacy (LDP). It was initially introduced by \cite{dwork} and has since been widely used to preserve privacy for individual data points as in \cite{kasiviswanathan2010learn}. This is achieved by introducing stochasticity in algorithm outputs to control their dependence on specific inputs. This stochasticity is generally quantified by a positive real number $\epsilon$. In this case, we write $\epsilon$-LDP instead of just LDP. The parameter $\epsilon$ quantifies the level of privacy protection in the local differential privacy mechanism, where a smaller value indicates stronger privacy protection. Do note, however, that we do not intend to have privacy-preserving models. For our present purposes, LDP is merely a technique to be applied to parts of our models in order to obtain risk bounds.

In this work, we impose the $\epsilon$-LDP condition on the gating networks of our models as a form of regularization. This method allows us to leverage the numerous benefits of the most complex architectures, such as neural networks, without compromising theoretical guarantees. By relying on LDP, we offer tight bounds on the risk of mixtures of experts models, provided they use the one-out-of-$n$ gating mechanism. Unlike the very few existing guarantees, our bounds depend only logarithmically on the number of experts we have, and the complexity of the gating network only appears in our bounds through the parameter $\epsilon$  of the LDP condition.  Existing bounds on mixtures of experts can be loose in certain cases—particularly when the outputs of the gating network do not depend, or depend only \textit{minimally}, on the input. The motivation for this work is to find a middle ground between input-independent aggregation mechanisms, such as in weighted majority vote classifiers, for which the risk is tightly bounded, and possibly highly input-dependent aggregation mechanisms, such as in usual mixtures of experts, which excel in practice. We aimed to encompass these two mechanisms by introducing a unified framework that quantifies the dependence of the gating mechanism on the input and integrates it into the theoretical guarantees. In our framework,
$\epsilon = 0$ gives input-independence, whereas
$\epsilon \to \infty$ gives arbitrary or uncontrolled input dependence.

\section{Preliminaries}
Let $\Xcal$ be the instance space, $\Ycal$ the label space, and $\Ycal'$ the output space (which can be different from $\Ycal$). As is usual in supervised learning, we assume that data $(x, y) \in \Xcal \times \Ycal$ are generated independently from an unknown probability distribution $\Dcal$. We consider a training set of $m$ examples $S=((x_1,y_1),\ldots,(x_m,y_m)) \sim \Dcal^m$ and a bounded loss function  $\ell\colon \Ycal' \times \Ycal \to [0,1]$.

\subsection{Mixtures of experts}\label{section:moe}
We consider classes $\Hcal_i$ of experts $h_i \colon \Xcal \to \Ycal'$ for $i = 1, \dots, n$. Let $\Gcal$ be a set of gating functions $\gb\colon \Xcal \to [0, 1]^n$ such that, given any $x \in \Xcal$, we have that $\sum_{i=1}^n g_i(x) = 1$, where $g_i(x)$ is the $i$-th component of $\gb(x)$. This means that each gating function defines a probability distribution on $[n] = \{ 1, \dots, n \}$ for each $x \in \Xcal$, where $g_i(x)$ is the probability of choosing the expert index $i$.  

In this work, a mixture of experts consists of $n$ experts, $\hb = (h_1, \ldots, h_n) \in \Hcal_1 \times \dots \times \Hcal_n$, a gating function $\gb$  $ \in \Gcal$ and a gating mechanism that combines the outputs of the experts and the output of the gating function to produce the final output. Our models use the stochastic one-out-of-$n$ 
 gating mechanism, as described in \cite{adaptive_mixture_of_local_experts}.
 It is defined as follows: to make a prediction with $(\gb, \hb) \in \Gcal \times \prod_{i = 1}^n \Hcal_i$ given an instance $x$, draw $i \sim \gb(x)$ and output $h_i(x)$. This stochastic predictor has risk and empirical risk defined by, respectively,
\[
R(\gb, \hb) = \ev_{(x, y) \sim \Dcal} \ev_{i \sim \gb(x)} \ell(h_i(x), y), \quad\text{and}\quad R_S(\gb, \hb) = \frac{1}{m} \sum_{j = 1}^m \ev_{i \sim \gb(x_j)} \ell(h_i(x_j), y_j).
\]
The preference for the one-out-of-$n$ gating mechanism over the $n$-out-of-$n$ mechanism in mixtures of experts is justified by its ability to induce sparsity and noise in the decision of which expert to use, enhancing computational efficiency and robustness to overfitting. This sparsity also offers scalability benefits, particularly in large-scale applications, where activating all experts for each input can lead to increased computational and memory requirements as explained in \cite{sparsly_gated_moe} and \cite{adaptive_mixture_of_local_experts}. Moreover, the one-out-of-$n$ mechanism enables us to derive tighter risk bounds for the models, where the bounds exhibit a logarithmic dependence on the number of experts, rather than the linear dependence often found in existing bounds. This shift from linear to logarithmic dependence significantly improves the tightness of the bounds, especially as the number of experts increases.

\subsection{Local Differential Privacy}
\begin{definition} 
    Let $\Ical$ be a finite set, consider a mechanism that produces an output $i \in \Ical$, given an input $x \in \Xcal$, with probability $\prob(i \mvert x)$, and let $\epsilon$ be a nonnegative real number. Then, the mechanism satisfies the $\epsilon$-Local Differential Privacy ($\epsilon$-LDP) property if and only if
    \begin{equation*} \label{eq:d-privacy}
        \prob(i \mvert x) \leq e^{\epsilon } \prob(i \mvert x') \quad \text{for all $x, x' \in \Xcal$ and all $i \in \Ical$}.
    \end{equation*}
\end{definition}

 Unless stated otherwise, we assume that each $\gb \in \Gcal$ satisfies $\epsilon$-LDP, for some fixed nonnegative real number $\epsilon$. Since we can interpret $\gb$ as a random mechanism that, given $x \in \Xcal$, selects $i \in [n]$ with probability $g_i(x)$, the condition of $\epsilon$-LDP amounts to the following: 
$$
g_i(x) \leq e^\epsilon g_i(x') \quad \text{for all $x, x' \in \Xcal$ and all $i \in [n]$.}
$$

Since $\epsilon$-LDP is an important condition for all of our theoretical results, we provide a practical way of obtaining gating functions satisfying $\epsilon$-LDP from an arbitrary set $\Fcal$ of bounded functions, in the form of the following theorem.
\begin{theorem}
\label{thm:privacy_bounded_function_softmax}
Let $b > 0$ and $\beta \geq 0$ be real numbers, and suppose that $\Fcal$ is a set of functions $\fb \colon \Xcal \to [-b, b]^n$. Let $\Gcal$ be the set of functions $\gb: \Xcal \to [0, 1]^n$ defined by 
\[
g_i(x)\ =\ \frac{\exp(\beta f_i(x)+c_i)}{\sum_{k=1}^n \exp(\beta f_k(x)+ c_k) }, \quad \text{where } \fb = (f_1, \ldots
    , f_n) \in \Fcal \text{ and } (c_1, \ldots
    , c_n) \in \Reals^n.
\]
Then, each $\gb \in \Gcal$ satisfies $4\beta b$-LDP.
\end{theorem}
\begin{proof}
    The proof is obtained by performing simple calculations, bounding the ratio \( g_i(x)/g_i(x') \), for all \( x, x' \in \Xcal \) and all \( i \in [n] \).
    The detailed proof is given in Appendix \ref{appendix:proofs}.
\end{proof}

In this work, we don't consider LDP of the whole model as a goal or objective to be achieved, though it would certainly be interesting for other work to do so. Rather, we use LDP as a tool to measure the dependence between the input and the output of the gating network, thereby bridging the gap between input-independent routing mechanisms and input-dependent ones.

\section{PAC-Bayesian bounds for mixtures of experts}

To apply the PAC-Bayes theory, we need to add a level of stochasticity to our predictors: instead of training experts $h_i$, we train probability measures $Q_i$ on each expert set $\Hcal_i$. For convenience, we write $Q = Q_1 \otimes \dots \otimes Q_n$. Now, putting everything together, a mixture of experts $(\gb, Q)$ makes predictions as follows: given $x \in \Xcal$, draw $i \sim \gb(x)$, then draw $h \sim Q_i$, and finally output $h(x)$. 
Such a predictor has risk and empirical risk defined by, respectively,
$$
R(\gb, Q) = \ev_{\hb \sim Q} R(\gb, \hb) \quad \text{and} \quad
R_S(\gb, Q) = \ev_{\hb \sim Q} R_S(\gb, \hb).
$$
Notice that, though probability distributions have replaced the individual experts, there is no need to define a probability distribution on the gating functions to get a PAC-Bayesian bound. Training a single gating function will do, and, remarkably, Lemma ~\ref{lemma:ldp_delta} below shows that it can be obtained from a very complicated function, such as a neural network, provided we impose $\epsilon$-LDP (for example, with Theorem~\ref{thm:privacy_bounded_function_softmax}). 

Finally, let us recall the notion of \emph{Kullback-Leibler (KL) divergence}. Given probability measures $Q_i$ and $P_i$ on $\Hcal_i$, it is defined by
\[
\KL(Q_i \mVert P_i) = \begin{dcases*}
\ev_{h \sim Q_i} \ln\frac{dQ_i}{dP_i}(h) & if $Q_i \ll P_i$  \\
\infty & otherwise,
\end{dcases*}
\]
where $dQ_i/dP_i$ is a Radon-Nikodym derivative. 

\begin{lemma}
 \label{lemma:ldp_delta}
    Consider mixtures of experts as defined in section \ref{section:moe} and provided with the \textit{one-out-of-$n$} routing mechanism. Let $\Delta: \Reals^2 \to \Reals$ be a convex function that is decreasing in its first argument and increasing in its second argument, and let $\epsilon$ be a nonnegative real number. Then, for any  $\gb \in \Gcal$ that satisfies the $\epsilon$-LDP property, for  any $Q = Q_1 \otimes \dots \otimes Q_n$ on $\Hcal_1 \times \dots \times \Hcal_n$, and for any $x' \in \Xcal$:
    \begin{equation*}
    \Delta\bigl(e^{\epsilon} R_S(\gb, Q), e^{-\epsilon} R(\gb, Q)\bigr) \leq \ev_{i \sim \gb(x')}\Delta\bigl(R_S(Q_{i}), R(Q_{i}) \bigr)
    \end{equation*}
    where $R(Q_{i}) = \ev_{x \sim \Dcal}  \ev_{h \sim Q_{i}} \ell(h(x), y)$ and $R_S(Q_{i}) = \frac{1}{m} \sum_{j=1}^m    \ev_{h \sim Q_{i}}  \ell(h(x_j), y_j)$.
\end{lemma}
\begin{proof}
    Since the gating function satisfies $\epsilon$-LDP, we have that $ e^{-\epsilon}g_i(x') \leq g_i(x) \leq e^{\epsilon}g_i(x')$ for all $x, x' \in \Xcal$ and all $i \in [n]$. It follows that $e^{\epsilon} R_S(\gb, Q) \geq \ev_{i \sim \gb(x')}R_S(Q_{i})$ and $e^{-\epsilon} R(\gb, Q) \leq \ev_{i \sim \gb(x')}R(Q_{i})$. 
    Given that $\Delta$ is decreasing in its first argument and increasing in its second argument, we find that
    \begin{equation*}
        \Delta\bigl(e^{\epsilon} R_S(\gb, Q), e^{-\epsilon} R(\gb, Q)\bigr) \leq \Delta\Bigl(\ev_{i \sim \gb(x')}R_S(Q_{i}), \ev_{i \sim \gb(x')}R(Q_{i}) \Bigr)
    \end{equation*}
    Since $\Delta$ is a convex function, we can apply \hyperref[thm-jensen]{Jensen's inequality} to the expression on the right-hand side, yielding the desired result. 
\end{proof}

Different choices of function $\Delta$ will allow us to obtain different PAC-Bayes bounds:
\begin{itemize}
    \item Let $\Delta(u, v) = v - u$. This is compatible with typical PAC-Bayes bounds on the difference between the true and empirical risks.
    \item Given $\lambda > 1/2$, let $\Delta$ be defined by $\Delta(u, v) = v - \frac{2\lambda}{2\lambda - 1}u$. This choice is compatible with a Catoni-type bound, as we will see below. 
    \item Let $\Delta$ be defined by $\Delta(u, v) = \kl(u \mVert v) = u \ln\frac{u}{v} + (1 - u) \ln\frac{1 - u}{1 - v}$. This choice is compatible with a Langford-Seeger-type bound. However, note that the function $\Delta$ defined here does not quite obey the hypotheses of lemma~\ref{lemma:ldp_delta}. Indeed, it is only defined for $(u, v) \in [0, 1]^2$, and only has the right monotonicity properties on the set $\{\mkern2mu (u, v) \in [0, 1]^2 \mid u \leq v \mkern2mu\}$. We can remedy those defects through small adjustments to the proof.
\end{itemize}

We prove a generalization bound of Catoni-type as an illustration of the machinery just described, for which we will need the following result:

\begin{theorem}[Theorem 2 in \cite{mcallester13}]
\label{thm-catoni}
    Let $\delta \in (0, 1)$ and $\lambda > 1/2$. Fix $i  \in [n]$, and let $P_i$ be a probability measure on $\Hcal_i$ (chosen without seeing the training data). Then, with probability at least $1 - \delta$ over the draws of $S$, for all probability measures $Q_i$ on $\Hcal_i$, we have that
    \[
    R(Q_i) \leq \frac{2\lambda}{2\lambda - 1}\biggl(R_S(Q_i) + \frac{\lambda}{m}\Bigl(\KL(Q_i \mVert P_i) + \ln\frac{1}{\delta}\Bigr)\biggr).
    \]
\end{theorem}

\begin{theorem}
\label{thm-catoni-ldp}
    Let $\delta \in (0, 1)$, $\epsilon \geq 0$, and $\lambda > 1/2$.  For each $i \in [n]$, let $P_i$ be a probability measure on $\Hcal_i$ (chosen without seeing the training data). Then, with probability at least $1 - \delta$ over the draws of $S$, for all probability measures $Q = Q_1 \otimes \dots \otimes Q_n$ on $\Hcal$, all $\gb \in \Gcal$ that satisfy $\epsilon$-LDP, and all $x' \in \Xcal$, we have that 
    \begin{equation*}
    R(\gb, Q) \leq \frac{2\lambda e^{\epsilon}}{2\lambda - 1}\biggl( e^{\epsilon} R_S(\gb, Q) + \frac{\lambda}{m}\Bigl(\ev_{i \sim \gb(x')} \KL(Q_i \mVert P_i) + \ln\frac{n}{\delta}\Bigr)\biggr). 
    \end{equation*}
\end{theorem}
\begin{proof}
    By $n$ applications of Theorem~\ref{thm-catoni}, we have that, for each $i \in [n]$, with probability at least $1 - \delta/n$, for all $Q_i$,
    \[
    R(Q_i) \leq \frac{2\lambda}{2\lambda - 1}\biggl(R_S(Q_i) + \frac{\lambda}{m}\Bigl(\KL(Q_i \mVert P_i) + \ln\frac{n}{\delta}\Bigr)\biggr).
    \]
    We can make all these inequalities (for each $i \in [n]$) hold simultaneously with a union bound. Now, applying Lemma~\ref{lemma:ldp_delta} with $\Delta(u, v) = v - \frac{2\lambda}{2\lambda - 1}u$, we find that, with probability at least $1 - \delta$, for all $Q$, all $\gb \in \Gcal$ and all $x' \in \Xcal$, we have that
    \begin{align*}
        e^{-\epsilon} R(\gb, Q) - \frac{2\lambda e^\epsilon}{2\lambda - 1} R_S(\gb, Q) &\leq \ev_{i \sim \gb(x')} \Bigl( R(Q_i) - \frac{2\lambda}{2\lambda - 1} R_S(Q_i) \Bigr) \\
        &\leq \frac{2\lambda^2}{(2\lambda - 1)m}\Bigl(\ev_{i \sim \gb(x')} \KL(Q_i \mVert P_i) + \ln\frac{n}{\delta}\Bigr). \qedhere
    \end{align*}
\end{proof}

We also give a bound of Langford-Seeger type, since they are generally recognized as among the tightest PAC-Bayes bounds available, and to prove the flexibility of our approach. 

\begin{theorem}
\label{thm-maurer-ldp}
    Let $\delta \in (0, 1)$, $\epsilon \geq 0$, and $m \geq 8$.  For each $i \in [n]$, let $P_i$ be a probability measure on $\Hcal_i$ (chosen without seeing the training data). Then, with probability at least $1 - \delta$ over the draws of $S$, for all probability measures $Q = Q_1 \otimes \dots \otimes Q_n$ on $\Hcal$, all $\gb \in \Gcal$ that satisfy $\epsilon$-LDP, and all $x' \in \Xcal$, we have that, either $R(\gb, Q) < e^{2 \epsilon} R_S(\gb, Q)$, or
    \[
    \kl(e^{\epsilon}R_S(\gb, Q) \mVert e^{-\epsilon}R(\gb, Q)) \leq \frac{1}{m}\Bigl(\ev_{i \sim \gb(x')}\KL(Q_i \mVert P_i) + \ln\frac{2n\sqrt{m}}{\delta}\Bigr).
    \]
\end{theorem}

\begin{proof} 
The proof, which is similar to that of Theorem~\ref{thm-catoni-ldp}, can be found in Appendix \ref{appendix:proofs}.
\end{proof}

\subsection{Comparison with other bounds}\label{sec:other-bounds}

Very few generalization bounds tailored specifically to mixtures of experts appear in the literature, and those we could find do not apply to mixtures of experts with the one-out-of-$n$ gating mechanism. 
We can, however, compare our bounds to those obtained by naively applying generic PAC-Bayes generalization bounds to mixtures of experts. In this case, we need to consider classifiers of the form $(Q_\Gcal, Q)$, where $Q_\Gcal$ is a probability measure on $\Gcal$, and $Q = Q_1 \otimes \dots \otimes Q_n$ is a probability measure on $\Hcal_1 \times \dots \times \Hcal_n$ as before. Then, note that
\[
\KL(Q_\Gcal \otimes Q_1 \otimes \dots \otimes Q_n \mVert P_\Gcal \otimes P_1 \otimes \dots \otimes P_n) = \KL(Q_\Gcal \mVert P_\Gcal) + \sum_{i = 1}^n \KL(Q_i \mVert P_i).
\]
This means that a generic PAC-Bayes bound applied to mixtures of experts will depend on the sum of the KL divergences corresponding to the gating functions and each of the experts. Obviously, this sum could be very large. By imposing $\epsilon$-LDP on the gating functions as in our approach, we can eliminate the stochasticity associated to the gating functions (that is, the probability measures $P_\Gcal$ and $Q_\Gcal$), and rid our bounds of the potentially large $\KL(Q_\Gcal \mVert P_\Gcal)$ term. Instead, it is $\epsilon$-LDP which controls our gating functions to ensure generalization. Furthermore, our bounds replace the sum of the KL divergences of the experts by a $\gb(x')$-weighted average, which means we can have many more experts with almost no penalty from the theoretical point of view. Indeed, our bounds only depend on the number $n$ of experts logarithmically, through the use of the union bound.

\section{Rademacher bounds for mixtures of experts}

Let us start with a slight modification of Lemma~\ref{lemma:ldp_delta}.

\begin{lemma}

 \label{lemma:ldp_delta2}
    We consider mixtures of experts as defined in section \ref{section:moe} and provided with the \textit{one-out-of-$n$} routing mechanism. Let $\Delta: \Reals^2 \to \Reals$ be a convex function that is decreasing in its first argument and increasing in its second argument, and let $\epsilon$ be a nonnegative real number. Then, for any  $\gb \in \Gcal$ that satisfies the $\epsilon$-LDP property, for  any $\hb\in\Hcal$ , and for any $x' \in \Xcal$:
    \begin{equation*}
    \Delta\bigl(e^{\epsilon} R_S(\gb, \hb), e^{-\epsilon} R(\gb, \hb)\bigr) \leq \ev_{i \sim \gb(x')}\Delta\bigl(R_S(h_{i}), R(h_{i}) \bigr)
    \end{equation*}
    where $R(h_{i}) = \ev_{x \sim \Dcal}  \ell(h_i(x), y)$ and $R_S(h_{i}) = \frac{1}{m} \sum_{j=1}^m    \ell(h_i(x_j), y_j)$.
\end{lemma}

\begin{proof}
The proof is similar to that of Lemma \ref{lemma:ldp_delta} and is provided in Appendix~\ref{appendix:proofs}.
\end{proof}

Let us now recall the following definition.

\begin{definition}[Rademacher complexity]
Given a space $\Hcal$ of predictors, a loss function $\ell$, and a data generating distribution $\Dcal$, the Rademacher complexity $\Rcal(\ell\circ\Hcal)$ is defined by
\[
\Rcal(\ell\circ\Hcal)\ =\ \ev_{S\sim \Dcal^m} \ev_{\sgb} \sup_{h\in\Hcal} \frac{1}{m} \sum_{j=1}^m \sg_j \ell(h(x_j), y_j) ,
\]
where $\sgb = (\sg_1,\ldots,\sg_m)$ is distributed uniformly on $\{-1, 1\}^m$. 
\end{definition}

Our main theorem will make use of the following well-known risk bound.

\begin{theorem}[Basic Rademacher risk bound] 
\label{thm-Rademacher}
    Given a $[0,1]$-valued loss function $\ell$, with probability at least $1-\dt$, for all $h\in\Hcal$, we have that
    \[
    R(h)\ \le R_S(h) + 2\Rcal(\ell\circ\Hcal) + \sqrt{\frac{2\ln(2/\dt)}{m}}\, .
    \]
\end{theorem}

\begin{theorem}
\label{thm-Rademacher-ldp}
    Let $\delta \in (0, 1)$ and $\epsilon \geq 0$. Given a $[0,1]$-valued loss function $\ell$, then, with probability at least $1 - \delta$ over the draws of $S$, for all $\hb\in \Hcal_1\times\dots\times\Hcal_n$, for all $\gb \in \Gcal$ that satisfy $\epsilon$-LDP, and all $x' \in \Xcal$, we have that
    \begin{equation*}
    R(\gb, \hb) \leq e^{\epsilon}\biggl( e^{\epsilon} R_S(\gb, \hb) + 2\ev_{i \sim \gb(x')} \Rcal(\ell\circ\Hcal_i) + \sqrt{\frac{2\ln(2n/\dt)}{m}}\,\biggr). \label{eq:thm-Rademacher-ldp}
    \end{equation*}
\end{theorem}
\begin{proof}
    By $n$ applications of Theorem~\ref{thm-Rademacher}, we have that, for each $i \in [n]$, with probability at least $1 - \delta/n$, for all $h_i\in \Hcal_i$,
    \[
    R(h_i) \leq R_S(h_i) + 2\Rcal(\ell\circ\Hcal_i) + \sqrt{\frac{2\ln(2n/\dt)}{m}}.
    \]
    We can make all these inequalities (for each $i \in [n]$) hold simultaneously with a union bound. Now, applying Lemma~\ref{lemma:ldp_delta2} with $\Delta(u, v) = v - u$, we find that, with probability at least $1 - \delta$, for all $\hb\in\Hcal_1 \times \dots \times \Hcal_n$, all $\gb \in \Gcal$ and all $x' \in \Xcal$, we have that
    \begin{align*}
        e^{-\epsilon} R(\gb, \hb) - e^\ep R_S(\gb, \hb) &\leq \ev_{i \sim \gb(x')} \bigl( R(h_i) - R_S(h_i) \bigr) \\
        &\leq \ev_{i \sim \gb(x')} \biggl( 2 \Rcal(\ell\circ\Hcal_i) + \sqrt{\frac{2\ln(2n/\dt)}{m}}\,\biggr). \qedhere
    \end{align*}
\end{proof}

Note, that the risk bound of Theorem~\ref{thm-Rademacher-ldp} depends only on the average Rademacher complexity of the classes of experts instead of the sum of their Rademacher complexities. Note also that, as in the previous section, the complexity of $\Gcal$ does not affect the risk bound. Finally, the risk bound does not hold uniformly for all values of $\ep$. However, by the union bound, the theorem holds for any fixed set $\{\ep_1,\dots,\ep_k\}$ if we replace $\dt$ by $\dt/k$. Consequently, this suggests a learning algorithm that minimizes $R_S(\gb,\hb)$ for $\ep\in \{\ep_1,\dots,\ep_k\}$. 

Also note that Lemma~\ref{lemma:ldp_delta2} allows us to obtain risk bounds for mixtures of experts as long as we have bounds on $\Delta\bigl(R_S(h_{i}), R(h_{i}) \bigr)$ which hold with high probability, whether they are based on Rademacher complexity, margins, VC dimension, or algorithmic stability.

\subsection{The need to use adaptive experts}

Following these theoretical results, we may be tempted to use a  gating network satisfying $\ep$-LDP to accomplish a learning task all by itself using non-adaptive experts, that is, experts $h_i$ each taking a constant value, no matter the input
. In that case, each Rademacher complexity $\Rcal(\ell\circ\Hcal_i)$ is zero and we can show that the bound of Theorem~\ref{eq:thm-Rademacher-ldp} can become vacuous in reasonable circumstances. 

Consider, for example, the binary classification case with the $0\mkern1mu$-$\mkern-1mu 1$ loss. In that case, we have two experts $h_{+1}$ and $h_{-1}$ such that $h_{+1}(x) = +1$ and $h_{-1}(x) = -1$ for all $x \in\Xcal$, and a gating network $\gb = (g_{+1}, g_{-1})$. Then, the following holds:
\begin{align*}
    R_S(\gb, \hb) &= \frac{1}{m} \sum_{j=1}^m\ev_{i \sim \gb(x_j)}\ell_{0\mkern1mu\text{-}\mkern-1mu 1}(h_i(x_j), y_j)\\
    &= \frac{1}{m} \sum_{j=1}^m\ev_{i \sim \gb(x_j)}\mathbf{1} (h_i(x_j)\neq y_j)\\
    &\geq \frac{1}{m} \sum_{j=1}^m\sum_{i \in \Ical} e^{-\epsilon} \max_{x' \in \Xcal} g_{i}(x')\mathbf{1} (h_i(x_j)\neq y_j), \quad\text{with } \Ical=\{+1, -1 \} \\
    &=e^{-\epsilon}\frac{1}{m} \sum_{j=1}^m  \max_{x' \in \Xcal} g_{-y_j}(x') .
\end{align*}
Under the assumption that the classes are balanced, meaning that the (marginal) probability of a positive label is equal to the (marginal) probability of a negative label, we have the following:
\begin{align*}
    \lim_{m\to\infty}\frac{1}{m} \sum_{j=1}^m  \max_{x' \in \Xcal} g_{-y_j}(x') &= \frac{1}{2} \Bigl(\max_{x' \in \Xcal} g_{-1}(x') + \max_{x' \in \Xcal} g_{+1}(x')\Bigr) \\
    &\geq \frac{1}{2} \max_{x' \in \Xcal}\bigl( g_{-1}(x') + g_{+1}(x')\bigr) = \frac{1}{2}.
\end{align*}

It follows that, 
in the limit $m\to\infty$, the risk bound of Theorem~\ref{eq:thm-Rademacher-ldp} for any $\gb$ has a value of at least $ e^{\epsilon}/2 \geq 1/2$. Consequently, the risk bound becomes large or even vacuous in this regime, highlighting the importance of having adaptive experts of finite complexity that can drive the empirical risk to zero when they are selected by the gating network.

\section{Experiments and results} \label{sec:experiments}
  In what follows, we consider mixtures of $n$ linear experts in binary classification tasks. Let $\Xcal = \Reals^d$ for some positive integer $d$. Let  $S$ be a training set of m examples. Each expert, denoted by $h_i$, where $i$ ranges from $1$ to $n$, is characterized by a weight vector $\mathbf{w}_i$. Given an input $\xb \in \Xcal$, the output of the expert $h_i$ is given by $h_i(\xb)= \inner{\wb_i, \xb}$. We use the probit loss function $\ell = \Phi$ (defined below), which can be seen as a smooth surrogate to the $0\mkern1mu$-$\mkern-1mu 1$ loss function, when it is used with an argument of the form $\frac{y \inner{\wb_i, \xb}}{\lVert\xb\rVert}$. In this case, $R(\gb, Q)$ and $R_S(\gb, Q)$  are given by:
\begin{equation*}
    R(\gb, Q) = \ev_{(\xb,y) \sim \Dcal} \ev_{i \sim \gb(\xb)} \Phi\Bigl(\frac{y \inner{\wb_i, \xb}}{\lVert\xb\rVert}\Bigr)
\end{equation*}
and
\begin{equation}\label{def:empirical_risk_analytic_expression}
    R_S(\gb, Q) = \frac{1}{m} \sum_{j=1}^m \sum_{i = 1}^n g_i(\xb_j) \Phi\Bigl(\frac{y_j \inner{\wb_i, \xb_j}}{\lVert\xb_j\rVert}\Bigr),
\end{equation}
where $\Phi(z) = \frac{1}{\sqrt{2\pi}} \int_{z}^{+\infty} e^{-t^2/2} \,dt$ 
 provides the probability that a standard normal random variable is greater than a given value $z$.

To illustrate the regularizing effect of the LDP condition, we carried out several experiments, on different datasets, by minimizing the empirical risk as defined in Equation \ref{def:empirical_risk_analytic_expression}. For all experiments, our models consist of mixtures of $n=100$ linear experts and a gating network. The gating network is a neural network having $2$ hidden layers. It is parameterized by weights $\Wb_1 \in \Reals^{64\times d}$, where $d$ is the dimension of input vectors, $\Wb_2\in \Reals^{64\times 64}$, and $\Wb_3\in \Reals^{n \times 64}$, and biases $ \bb_1 \in \Reals^{64}$, $ \bb_2 \in \Reals^{64}$ and $ \bb_3 \in \Reals^{n}$. Given an input $\xb \in \Reals^d$, the output of the gating network $\gb(\xb) = (g_1(\xb), \ldots, g_n(\xb))$ is computed as follows: first, we compute $\fb_0(\xb)  = \tanh(\Wb_2\ReLU(\Wb_1\xb + \bb_1)+ \bb_2)$. Then, when we want the $\epsilon$-LDP condition to be satisfied, we ensure that the outputs are between $-\epsilon/4$ and $\epsilon/4$:
\[
\fb(\xb) = 
\begin{cases}
   \frac{\epsilon\Wb_3\fb_0(\xb)}{4\lVert\fb_0(\xb)\rVert\lVert\Wb_3\rVert_F} & \text{if the gating network must satisfy $\epsilon$-LDP} \\
  \Wb_3\fb_0(\xb) & \text{otherwise}.
\end{cases}
\]
Note that $\tanh$ is the hyperbolic tangent activation function, ReLU the Rectified Linear Unit function, $\lVert\Wb_3\rVert_F$ the Frobenius norm of the matrix $\Wb_3$, and $\lVert\fb_0(\xb)\rVert$ the euclidean norm of the vector $\fb_0(\xb)$.
Indeed, if we let $\Wb_3^i$ denote the $i$-th row of $\Wb_3$, then the $i$-th component of $\Wb_3 \fb_0(\xb)$ is 
\[
\Wb_3^i \cdot \fb_0(\xb) \leq \lVert\Wb_3^i \rVert\lVert\fb_0(\xb)\rVert \leq \lVert\Wb_3 \rVert_F\lVert\fb_0(\xb)\rVert,
\]
by the Cauchy-Schwarz inequality and the definition of the Frobenius norm. The reason we use the Frobenius norm instead of directly using $\lVert\Wb_3^i\rVert$ is to preserve the proportions between the components of $\Wb_3 \fb_0(\xb)$ when setting up $\epsilon$-LDP. 

The final output of the gating network is given by
\[
g_i(\xb)\ =\frac{\exp(f_i(\xb)+(\bb_3)_i)}{\sum_{k=1}^n \exp(f_k(\xb)+ (\bb_3)_k) } \quad\text{for all}\quad i \in [n].
\]

In our experiments, we ran the Stochastic Gradient Descent algorithm 10 times with a learning rate fixed to $0.1$. In each experiment, we trained the model for 1000 epochs, except for the MNIST dataset, where the training duration was shortened to 300 epochs due to dataset size. We allocated approximately $75\%$ of the data to the training set and the remaining $25\%$ to the test set. At the outset of each experiment, the weights of our neural networks were reinitialized to ensure a fresh starting point. After each training run, we computed both the training and test loss values to evaluate the model's performance. We first ran the training without imposing any constraints on the gating network, except for the architecture. Then, we ran several experiments with a gating mechanism satisfying $\epsilon$-LDP, with $\epsilon \in \{ 0.5, 2, 4, 5, 10\}$. A summary of the results is shown in Table \ref{tab:ldp_regularization_effect}. 
One can observe that regularization with $\epsilon$-LDP improves results in practice, and this regularization is even more evident when the models employing a gating network not satisfying LDP overfit heavily, as in the Breast Cancer and Heart experiments. The regularization effect is slightly less pronounced on MNIST, where the overfitting is not as severe as with the previous datasets. We can also observe the importance of choosing the right hyperparameter $\epsilon$. Indeed, if the value is too small, the output of the gating network becomes insufficiently dependent on the input $\xb$. In this case, the experts have to do all the work, and the gating network does not allow them to specialize in well-defined subsets of the instance space. This makes our model closer to a weighted sum of linear classifiers and significantly reduces its performance. Conversely, if $\epsilon$ is overly large, our model tends towards a situation where the LDP condition does not hold, making it prone to overfitting.

\begin{table}[ht]
    \centering
    \caption{Experiment results for mixtures of 100 linear models applied to binary classification tasks: Ads, Breast Cancer \citep{misc_breast_cancer_14}, Heart \citep{misc_heart_disease_45} and MNIST \citep{deng2012mnist}. The objective is to minimize the empirical risk as defined in Equation \ref{def:empirical_risk_analytic_expression}. We report the mean training loss ($R_S$) and mean test loss ($R_T$), averaged over ten runs, along with their associated standard deviations.\protect\footnotemark }
    
\begin{tabular}{ c c c c c c c c c  c  }
 \hline
 & & \multicolumn{1}{ c }{ }&\multicolumn{5}{c }{MoE with a gating network satisfying $\epsilon$-LDP} \\
 \cline{4-8}
 \multicolumn{1}{c}{Dataset }& Risk& \multicolumn{1}{c}{No LDP }  &\multicolumn{1}{c}{$\epsilon = 0.5$ } & 
 \multicolumn{1}{c}{$\epsilon = 2$}& \multicolumn{1}{c}{$\epsilon = 4$ } & \multicolumn{1}{c}{$\epsilon = 5$ } & \multicolumn{1}{c}{$\epsilon = 10$}\\
 \hline 
  Ads &$R_S$ & 0.02425 & 0.13854 &0.01829 & 0.05288 & 0.06459 & 0.02811 \\
&$\pm$ & 0.00499 & 0.00261 &0.00216 & 0.05543 & 0.05821 & 0.03648 \\
\cline{2-8}
&$R_T$ & 0.03822 & 0.13051 &\textbf{0.03206} & 0.06693 & 0.07757 & 0.04384 \\
&$\pm$ & 0.00696 & 0.01138 &\textbf{0.00564} & 0.05276 & 0.05822 & 0.03501 \\
 \hline 
 Breast  &$R_S$ &0.00780 & 0.04520 & 0.01252 & 0.01062 & 0.01089 & 0.01207 \\
Cancer&$\pm$ & 0.00347 & 0.00426 & 0.00182 & 0.00286 & 0.00193 &0.00181 \\
\cline{2-8}
&$R_T$ & 0.03617 & 0.04930 & 0.03238 & 0.03297 & 0.02942 &\textbf{0.02604} \\
&$\pm$ & 0.01505 & 0.01244 & 0.01349 & 0.01379 &0.00948 & \textbf{0.01277} \\
 \hline
 Heart &$R_S$ &0.00001 & 0.03524 & 0.00015 & 0.00010 & 0.00009 & 0.00013 \\
&$\pm$ &0.00000 & 0.00487 & 0.00002 & 0.00001 & 0.00001 & 0.00006 \\
\cline{2-8}
&$R_T$ & 0.00029 & 0.03962 &\textbf{0.00026} & 0.00026 & 0.00032 & 0.00032 \\
&$\pm$ & 0.00065 & 0.01013 &\textbf{0.00014} & 0.00033 & 0.00030 & 0.00032 \\
 \hline
 MNIST  &$R_S$ & 0.00525 & 0.00558 & 0.00529 &0.00504 & 0.00536 & 0.00523 \\
0 vs 8 &$\pm$ &0.00029 & 0.00059 & 0.00044 & 0.00031 & 0.00031 & 0.00032 \\
\cline{2-8}
&$R_T$ & 0.00844 & 0.00869 & 0.00815 & 0.00864 &\textbf{0.00769} & 0.00802 \\
&$\pm$ & 0.00103 & 0.00109 & 0.00131 & 0.00165 & \textbf{0.00144} &0.00067  \\
 \hline
 MNIST  &$R_S$ & 0.00287 & 0.00330 & 0.00289 &0.00285 & 0.00298 & 0.00286 \\
1 vs 7 &$\pm$ & 0.00024 & 0.00033 & 0.00028 & 0.00025 & 0.00023 &0.00013 \\
\cline{2-8}
&$R_T$ & 0.00501 & 0.00485 & 0.00501 & 0.00518 &\textbf{0.00450} & 0.00526 \\
&$\pm$ &0.00042 & 0.00101 & 0.00093 & 0.00098 & \textbf{0.00101} & 0.00066 \\
 \hline
  MNIST  &$R_S$ & 0.01419 & 0.01509 & 0.01388 & 0.01396 & 0.01440 &0.01154 \\
5 vs 6&$\pm$ & 0.00046 & 0.00057 &0.00038 & 0.00051 & 0.00056 & 0.00336 \\
\cline{2-8}
&$R_T$ & 0.02195 & 0.02131 & 0.02206 & 0.02236 & 0.02072 &\textbf{0.01852} \\
&$\pm$ &0.00111 & 0.00160 & 0.00185 & 0.00269 & 0.00229 & \textbf{0.00518} \\
 \hline
\end{tabular}
\label{tab:ldp_regularization_effect}
\end{table}
\footnotetext{If $N$ denotes the number of runs, $R_k$ denotes the training or test empirical risk during the $k$-th run, and $\bar{R}$ denotes the average, then standard deviation is given by $\sqrt{\frac{1}{N} \sum_{k = 1}^N (R_k - \bar{R})^2}$.}
Note that our experiments are executed on GPUs in order to parallelize computations and take advantage of the sparsity of our model, but they can also be performed without GPUs. The duration of experiments can range from a few minutes for small datasets such as Breast Cancer to around 3 hours for large datasets like MNIST.

\section{Conclusion}

In this work, we introduce a new way to regularize mixtures of experts. We provide both theoretical and algorithmic contributions in this regard.
Our approach offers a significant advantage in that it allows us to harness the remarkable performance of neural networks by using them as gating networks, without being constrained by their architecture or their complexity from the theoretical point of view. By imposing LDP, we obtain nonvacuous bounds on the mixture of experts' risk. Our bounds can become significantly tighter than those presented in section \ref{sec:other-bounds} and those presented in \cite{risk_bound_moe_azran_ron}, especially in cases where the empirical risk is close to zero and $\epsilon < \ln n$. However, as the empirical risk is multiplied by $e^\epsilon$ in our bounds, they can become loose when $\epsilon$ is large and the empirical risk is significant. 

Even though the $\epsilon$-LDP condition is easy to set up, a challenge arises in striking a balance between the parameter $\epsilon$ and the $\KL$ divergence or the Rademacher complexity of our experts. Our method introduces an extra hyperparameter $\epsilon$ to optimize but does not provide theoretical guidance on configuring it. This forces us to navigate a trade-off between the value of $\epsilon$, which measures the extent to which the output of the gating network can depend on a given $x \in \Xcal$, and the complexity of our experts, which reflects how well our model captures the data distribution. Finding the right balance requires empirical testing and careful consideration and can open up new avenues of study in the future.

\begin{ack}
This work is supported by the DEEL Project CRDPJ 537462-18 funded by the Natural Sciences and Engineering Research Council of Canada (NSERC) and the Consortium for Research and Innovation in Aerospace in Québec (CRIAQ), together with its industrial partners Thales Canada inc, Bell Textron Canada Limited, CAE inc and Bombardier inc.\footnote{\url{https://deel.quebec}}

\end{ack}

\bibliographystyle{unsrtnat}
\bibliography{references}
\appendix
\newpage

\section{Proofs and auxiliary results}\label{appendix:proofs}

\begin{proof}[Proof of theorem~\ref{thm:privacy_bounded_function_softmax}]

Given $x \in \Xcal$, let $Z(x) = \sum_{i = 1}^n \exp(\beta f_i(x) + c_i)$, for convenience.

For all $x, x' \in \Xcal$ and all $i \in [n]$, we have that
\begin{align*}
\frac{g_i(x)}{g_i(x')} &= \exp\bigl(\beta (f_i(x) - f_i(x'))\bigr) \frac{1}{Z(x)} \sum_{k = 1}^n \exp(\beta f_k(x') + c_k) \\
&= \exp\bigl(\beta (f_i(x) - f_i(x'))\bigr) \frac{1}{Z(x)}  \sum_{k = 1}^n \exp(\beta f_k(x) + c_k) \exp\bigl(\beta (f_k(x') - f_k(x))\bigr) \\
&\leq \max_{i \in [n];\, x_1, x_2 \in \Xcal} \exp\bigl(2\beta (f_i(x_1) - f_i(x_2))\bigr) \frac{1}{Z(x)}  \sum_{k = 1}^n \exp(\beta f_k(x) + c_k) \\
&\leq \exp(4\beta b). \qedhere
\end{align*}
\end{proof}

\begin{theorem}[Jensen's inequality, proposition 1.1 in \cite{perlman74}]
\label{thm-jensen}
    Let $\Omega$ be a probability space, let $A$ be a convex subset of $\Reals^k$, let $X \colon \Omega \to A$ be an integrable vector-valued random variable, and let $\phi: A \to \Reals$ be a convex function. Then, $\ev\mkern-1mu X \in A$, and $\phi(\ev\mkern-1mu X) \leq \ev \phi(X)$ (in particular, the right-hand side of this inequality exists, though it may be infinite). 
\end{theorem}

\begin{theorem}[Theorem 5 in \cite{maurer04}]
\label{thm-maurer}
    Let $\delta \in (0, 1)$ and $m \geq 8$.  Fix $i \in [n]$, and let $P_i$ be a probability measure on $\Hcal_i$ (chosen without seeing the training data). Then, with probability at least $1 - \delta$ over the draws of $S$, for all probability measures $Q_i$ on $\Hcal_i$, we have that
    \[
    \kl(R_S(Q_i) \mVert R(Q_i)) \leq \frac{1}{m}\Bigl(\KL(Q_i \mVert P_i) + \ln\frac{2\sqrt{m}}{\delta}\Bigr).
    \]
\end{theorem}

\begin{proof}[Proof of theorem~\ref{thm-maurer-ldp}]
     As remarked earlier, the function $(u, v) \xrightarrow{} \kl(u \mVert v): [0, 1]^2 \to \Reals$ does not exactly satisfy the hypotheses of lemma~\ref{lemma:ldp_delta}, but it is convex. Moreover, on $\{\mkern2mu (u, v) \in [0, 1]^2 \mvert u \leq v \mkern2mu\}$, it is decreasing in its first argument and increasing in its second argument. Also note that, assuming that $R(\gb, Q) \geq e^{2 \epsilon} R_S(\gb, Q)$, then we also have the following inequalities:
    \[
    0 \leq \ev_{i \sim \gb(x')} R_S(Q_i) \leq e^{\epsilon} R_S(\gb, Q) \leq e^{-\epsilon} R(\gb, Q) \leq \ev_{i \sim \gb(x')} R(Q_i) \leq 1.
    \]
    It follows that
    \begin{align*}
        \kl(e^{\epsilon} R_S(\gb, Q) \mVert e^{-\epsilon} R(\gb, Q)) &\leq
        \kl\Bigl(\ev_{i \sim \gb(x')} R_S(Q_i) \mkern3mu\Big\Vert\mkern3mu e^{-\epsilon} R(\gb, Q)\Bigr) \\
        &\leq \kl\Bigl(\ev_{i \sim \gb(x')} R_S(Q_i) \mkern3mu\Big\Vert\mkern3mu \ev_{i \sim \gb(x')} R(Q_i)\Bigr),
    \end{align*}
    and therefore
    \[
    \kl(e^{\epsilon} R_S(\gb, Q) \mVert e^{-\epsilon} R(\gb, Q)) \leq \ev_{i \sim \gb(x')}\kl\bigl(R_S(Q_i) \mVert R(Q_i)\bigr)
    \]
    by \hyperref[thm-jensen]{Jensen's inequality}.
    Now, by theorem~\ref{thm-maurer}, for a fixed $i$, with probability at least $1 - \delta/n$, we have that
    \[
    \kl(R_S(Q_i) \mVert R(Q_i)) \leq \frac{1}{m}\Bigl(\KL(Q_i \mVert P_i) + \ln\frac{2n\sqrt{m}}{\delta}\Bigr).
    \]
    We can make the above inequality hold for all $i \in [n]$ simultaneously with the union bound.
    Then, with probability at least $1 - \delta$, for all $(\gb, Q)$, given that $R(\gb, Q) \geq e^{2 \epsilon} R_S(\gb, Q)$, we have that
    \[
    \kl(e^{\epsilon}R_S(\gb, Q) \mVert e^{-\epsilon}R(\gb, Q)) \leq \frac{1}{m}\Bigl(\ev_{i \sim \gb(x')}\KL(Q_i \mVert P_i) + \ln\frac{2n\sqrt{m}}{\delta}\Bigr). \qedhere
    \]
\end{proof}

\begin{proof}[Proof of Lemma \ref{lemma:ldp_delta2}]
    Since the gating function satisfies $\epsilon$-LDP, we have that $ e^{-\epsilon}g_i(x') \leq g_i(x) \leq e^{\epsilon}g_i(x')$ for all $x, x' \in \Xcal$ and all $i \in [n]$. It follows that $e^{\epsilon} R_S(\gb, \hb) \geq \ev_{i \sim \gb(x')}R_S(h_{i})$ and $e^{-\epsilon} R(\gb, \hb) \leq \ev_{i \sim \gb(x')}R(h_{i})$. 
    Given that $\Delta$ is decreasing in its first argument and increasing in its second argument, we find that
    \begin{equation*}
        \Delta\bigl(e^{\epsilon} R_S(\gb, \hb), e^{-\epsilon} R(\gb, \hb)\bigr) \leq \Delta\Bigl(\ev_{i \sim \gb(x')}R_S(h_{i}), \ev_{i \sim \gb(x')}R(h_{i}) \Bigr)
    \end{equation*}
    Since $\Delta$ is a convex function, we can apply \hyperref[thm-jensen]{Jensen's inequality} to the expression on the right-hand side, yielding the desired result. 
\end{proof}

\end{document}